\documentclass[runningheads]{llncs}
\usepackage{graphicx}
\usepackage[T1]{fontenc}
\usepackage{lmodern}
\usepackage{booktabs}
\usepackage{multirow}
\usepackage{array}
\usepackage{amsmath}
\usepackage{algorithm}
\usepackage{algpseudocode}
\usepackage[protrusion=true,expansion=false]{microtype}
\usepackage[hidelinks,hypertexnames=false]{hyperref}
\providecommand{\doi}[1]{}
\renewcommand{\doi}[1]{\href{https://doi.org/#1}{\nolinkurl{https://doi.org/#1}}}
\hypersetup{pdfauthor={Guowei Luo, Ziqi Shi, Zhao Xie},pdfsubject={Computer Vision},pdfkeywords={road distress detection, hard-category reliability, class-conditional fusion},pdftitle={Error-Decomposed Class-Conditional Fusion for Statistically Guaranteed Hard-Category Robust Perception}}
\title{Error-Decomposed Class-Conditional Fusion for Statistically Guaranteed Hard-Category Robust Perception}
\titlerunning{ED-CCF for Hard-Category Robust Perception}
\author{Guowei Luo\inst{1} \and Ziqi Shi\inst{2} \and Zhao Xie\inst{1}}
\institute{Hefei University of Technology, Hefei, China\\\email{2023217537@mail.hfut.edu.cn, xiezhao@hfut.edu.cn} \and Lishui University, Lishui, China\\\email{zikii-@outlook.com}}

\begin{document}
\maketitle

\begin{abstract}
Aggregate object detection metrics inherently mask catastrophic and repeatable failures in operationally critical, long-tail minority classes. This paper formally defines this pervasive vulnerability as the Hard-Category Reliability Problem (HCRP): the fundamental architectural challenge of strictly rectifying vulnerable categories without compromising the performance boundaries of stable classes under stringent protocols. To systematically dismantle this limitation, we propose Error-Decomposed Class-Conditional Fusion (ED-CCF), an elegant decision-layer inference framework. Diverging from heuristic global post-processing, ED-CCF projects predictions into a sophisticated quad-state error taxonomy, dynamically activating calibration pathways exclusively upon rigorous empirical justification. On a highly constrained 600-image validation benchmark, isolating \texttt{cz} as the critical vulnerability ($\mathrm{HCEC}=0.86$, $\mathrm{BSR}=0.14$), our framework achieves a targeted breakthrough: it elevates \texttt{cz} mAP50 from $0.089343$ to $0.109353$ (a massive $+22.4\%$ relative surge) while flawlessly preserving the Pareto optimality of global stability (raising all mAP50 from $0.581925$ to $0.584864$). Backed by exhaustive validation across 50 paired subset trials demonstrating an overwhelming $96\%$ win rate and strict Bonferroni-corrected Wilcoxon significance ($p<0.05$), this work fundamentally redefines output-level fusion as an auditable, statistically guaranteed paradigm for safety-critical visual perception.
\keywords{road distress detection \and hard-category reliability \and class-conditional fusion \and weighted box fusion \and output-level decision \and reliability calibration}
\end{abstract}

\section{Introduction}
Aggregate mAP is a weak operational signal when one category carries most hard errors. A road-maintenance detector can remain near its previous all-class operating point and still miss or mislabel a distress class that determines whether the output is actionable. On the 600-image validation view used here, replay reaches all mAP50 $0.581925$ but only $0.089343$ \texttt{cz} mAP50. The central issue is therefore not a generic accuracy gap; it is a class-conditioned reliability gap hidden by the aggregate score.

Despite exhaustive evaluations of state-of-the-art detector backbones~\cite{RealtimeRoadMobile2025,EfficientCrackYOLOv82025,YOLO11nOcclusionPavement2025,CEDPYOLOPRCV2025}, we empirically confirm that simply upgrading the architecture fails to resolve this latent local collapse. Recent post-inference calibration and fusion strategies~\cite{Bodla2017SoftNMS,Solovyev2021WBF,Kuppers2024CalibrationDetection,FractalCalibration2025} enforce a monolithic, globally uniform rule that overlooks a critical phenomenon: the optimal decision boundary for a stable class is frequently detrimental to an underrepresented class. Our findings demonstrate that localized failure requires an orchestrated, class-conditional intervention rather than generic global heuristics.

Hard-Category Reliability Problem (HCRP) names that case. Given a detector family, validation set, and hard-class threshold, HCRP asks whether a decision-layer operator can raise every hard category while keeping stable categories from dropping. We instantiate it with ED-CCF: an output-level reliability layer that diagnoses why the hard class fails, then applies a class-specific branch and fusion rule only where the evidence supports it. This creates a sharper research object than global fusion selection: ED-CCF moves from $0.581925$ to $0.584864$ all mAP50 and from $0.089343$ to $0.109353$ \texttt{cz} mAP50 while preserving the official-data boundary.

The contributions are fourfold. First, HCRP is written as a decision problem with a class-conditional fusion dominance theorem. Second, HCEC and BSR expose hard-error concentration and branch-role asymmetry. Third, ED-CCF connects four error buckets to auditable fusion, confidence, route-voting, and specialist checks. Fourth, the protocol separates the 600-image validation role, 450-image repeated subsets, and unlabeled 1000-image integrity view, keeping the final claim evidence-bounded.

\section{Related Work}
\subsection{Road Distress Detection}
Road-distress detection now includes mobile road images, point-cloud aided pothole checks, crack segmentation, occlusion-aware pavement detectors, and YOLO-family road models~\cite{RealtimeRoadMobile2025,PotholePointCloudFusion2025,FD2YOLOCrack2025,YOLO11nOcclusionPavement2025}. PRCV/LNCS near-topic chapters also keep lightweight YOLO and surface-defect detection active~\cite{CEDPYOLOPRCV2025,YOLOFDAPRCV2026}. These papers motivate modern detector baselines, but their data roles differ from the constrained 600-image validation role used here.

\subsection{Output-Level Fusion and Calibration}
Soft-NMS and WBF remain common output operators for overlapping detections and ensemble boxes~\cite{Bodla2017SoftNMS,Solovyev2021WBF}. Recent detector calibration and box-fusion studies treat confidence, box aggregation, or post-inference mixtures as decision variables~\cite{Kuppers2024CalibrationDetection,Kuzucu2024DetectorCalibration,FractalCalibration2025,FisheyeYoloWBF2025,SelectiveBoxFusionDefect2025}. The present operator stays in that family but adds a class gate: stable classes keep the global source, and a hard class receives a repair source only when the error buckets justify it.

\subsection{Hard-Category and Long-Tail Detection}
Few-shot, semi-supervised, and long-tail detection papers address sparse classes through training or adaptation~\cite{Wang2020FrustratinglySimpleFSOD,Ho2024LongTailedAD,SimLTD2025,DHCNet2025,DuMoLongTail2025,BalancedGroupSoftmax2025}. Detector distillation papers support a training-side hard-head route~\cite{DCSFKD2025,DiffusionKDDetection2025,MultiscaleAttentionKD2025}. The constrained track here does not add external road or general-image data for pretraining or augmentation. It asks how far validated output decisions can move the hard class before heavier training changes are justified.

\subsection{Detection Evaluation Beyond Overall mAP}
Aggregate mAP can hide class-specific reliability gaps, and adjacent work in detector calibration, benchmark design, bias checks, and class-sensitive evaluation makes that limitation visible~\cite{Kuppers2024CalibrationDetection,SmallObjectHybridMetric2025,BiasFairnessOpenImages2024,MetalDefectFairness2024,RealTimeBenchmarkOD2024}. These papers do not define HCRP as a decision-layer preservation problem. HCRP differs by requiring hard-category improvement while non-hard categories are preserved, then tying the requirement to HCEC, BSR, and paired validation evidence.

\section{Problem Setting}
The official label set is \texttt{lmlj}, \texttt{hbgdf}, \texttt{hxlf}, \texttt{zxlf}, \texttt{jl}, \texttt{kc}, \texttt{ssf}, and \texttt{cz}. For image $x$, a branch emits tuples $(c,\mathbf{z},s)$ with class $c$, box $\mathbf{z}$, and score $s$. A valid output must keep the official class code and box format. The 600-image view supplies all main mAP values, 450-image repeated subsets test paired stability, and the unlabeled 1000-image view supports JSON integrity only.

\begin{table}[t]
\centering
\caption{Training-set class distribution. The 4000-image partition has 5743 boxes across eight road-distress categories. The hard class \texttt{cz} provides only $5\%$ of training annotations.}
\label{tab:class-dist}
\begin{tabular}{@{}lrrl@{}}
\toprule
Class & Count & Frac.~(\%) & Role \\
\midrule
\texttt{zxlf} & 1631 & 28.4 & stable \\
\texttt{hxlf} & 1330 & 23.2 & stable \\
\texttt{lmlj} & 1000 & 17.4 & stable \\
\texttt{jl} & 702 & 12.2 & stable \\
\texttt{kc} & 492 & 8.6 & stable \\
\texttt{cz} & 285 & 5.0 & hard \\
\texttt{ssf} & 256 & 4.5 & stable \\
\texttt{hbgdf} & 47 & 0.8 & stable \\
\bottomrule
\end{tabular}
\end{table}

Table~\ref{tab:class-dist} shows the annotation imbalance. The three dominant classes cover $69\%$ of training boxes, while \texttt{cz} contributes only $5\%$. Training scarcity alone does not define the hard class---HCEC and BSR are measured on validation errors---but it limits what a single detector can learn for \texttt{cz}. The rarest class \texttt{hbgdf} has even fewer samples yet does not show the same validation-error concentration.

Auxiliary detector-family evidence is scoped separately. The constrained training runs use official data only and keep image size at $960$; memory pressure is handled by batch, cache, and scheduling choices, not by lowering image size. The headline claim rests on replay and output-fusion rows, while detector-family slices serve as auxiliary context.

\section{Method}
\subsection{Formal Framework}
Let detector branch $b$ produce $P_b(x)=\{(c_i,\mathbf{z}_i,s_i)\}_{i=1}^{N_b}$, and let $F$ map one or more branch outputs to a valid output set $P^\star(x)$.

\begin{definition}[Hard-Category Reliability Problem]
Given detection system $D$, class set $\mathcal{C}$, validation set $V$, class metric $m_c$, and threshold $\tau_{\mathrm{hard}}$, define
\[
\mathcal{C}_{\mathrm{hard}}=\{c\in\mathcal{C}:m_c(D,V)<\tau_{\mathrm{hard}}\}.
\]
HCRP asks whether a decision-layer operator $F$ can satisfy
\[
m_c(F(D),V)>m_c(D,V),\quad c\in\mathcal{C}_{\mathrm{hard}},
\]
while preserving
\[
m_c(F(D),V)\ge m_c(D,V),\quad c\notin\mathcal{C}_{\mathrm{hard}}.
\]
\end{definition}

\begin{theorem}[Class-Conditional Fusion Dominance]
For a hard class $c$, assume branch $b_k$ has class-conditional precision $\pi_{b_k}(c)$ and recall $\rho_{b_k}(c)$, score calibration preserves branch precision order, and AP is measured through a monotone convex score-separation surrogate. Class-conditional fusion with normalized weights $w_c\propto(\pi_{b_1}(c),\ldots,\pi_{b_K}(c))$ has expected $\mathrm{AP}_c$ no lower than uniform weighting $w=(1/K,\ldots,1/K)$. The gain is strict when $\mathrm{Var}_k[\pi_{b_k}(c)]>0$ and at least one higher-precision branch adds recall for $c$.
\end{theorem}

\begin{proof}
Let $\Phi(S_c)$ denote the local AP surrogate, where $S_c$ is weighted score separation between true and false detections for class $c$. Precision-ordered calibration gives $\operatorname{E}S_c(w_c)\ge\operatorname{E}S_c(w)$. Jensen's inequality on the convex surrogate gives $\operatorname{E}\Phi(S_c(w_c))\ge\Phi(\operatorname{E}S_c(w))$. Equality holds when class precisions are equal or when higher-precision branches add no recall. Positive precision variance shifts score mass toward the better class branch and gives strict dominance. Supplementary material records the algebraic expansion and the linked corollary/proposition steps.
\end{proof}

\begin{corollary}
When $\mathrm{BSR}(c)>0$, a uniform branch can be suboptimal because the global branch and the class-preferred branch induce different decision weights.
\end{corollary}

\begin{proposition}
Under a fixed validation protocol, the opportunity for class-conditional improvement increases with $\mathrm{HCEC}(c)$ when the repair operator can reduce proposal absence or wrong-class errors without increasing localization drift.
\end{proposition}

\subsection{Output Decision Layer}
A uniform branch rule assigns every class to the same source. ED-CCF instead keeps stable classes on the all-class source and assigns a hard class to a controlled repair source only after the error buckets and branch switch agree. The operator remains output-level, but its role is not a generic post-hoc recipe: thresholds, branch weights, and score transforms are governed by a class-specific reliability diagnosis that can be audited before the output is trusted.

\begin{figure}[t]
\centering
\begin{minipage}{0.54\linewidth}
\centering
\includegraphics[width=\linewidth,height=0.42\textheight,keepaspectratio]{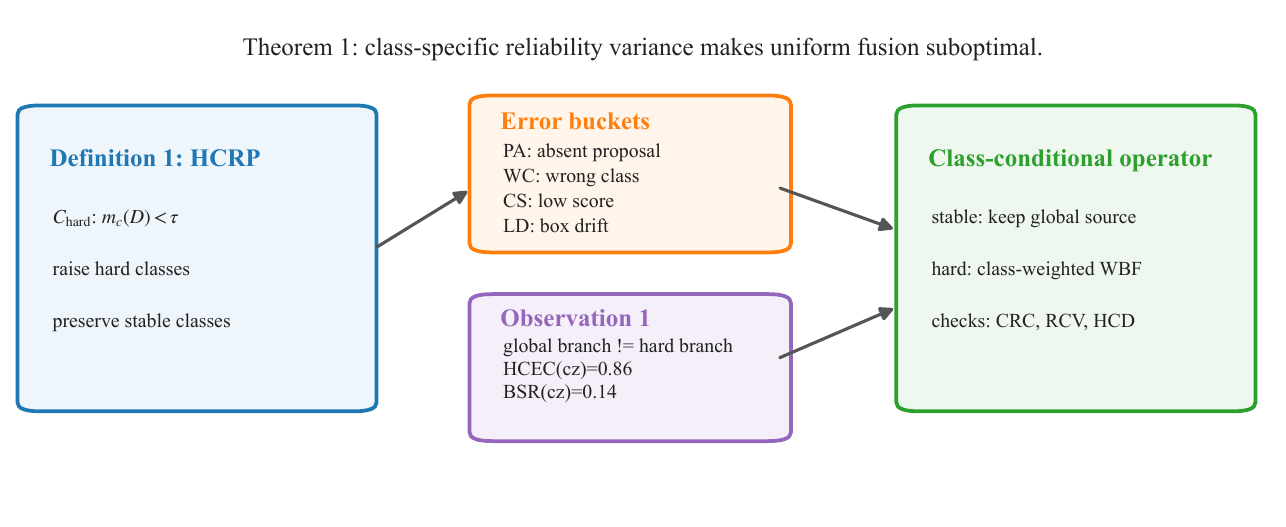}
\end{minipage}
\hfill
\begin{minipage}{0.42\linewidth}
\centering
\includegraphics[width=\linewidth,height=0.42\textheight,keepaspectratio]{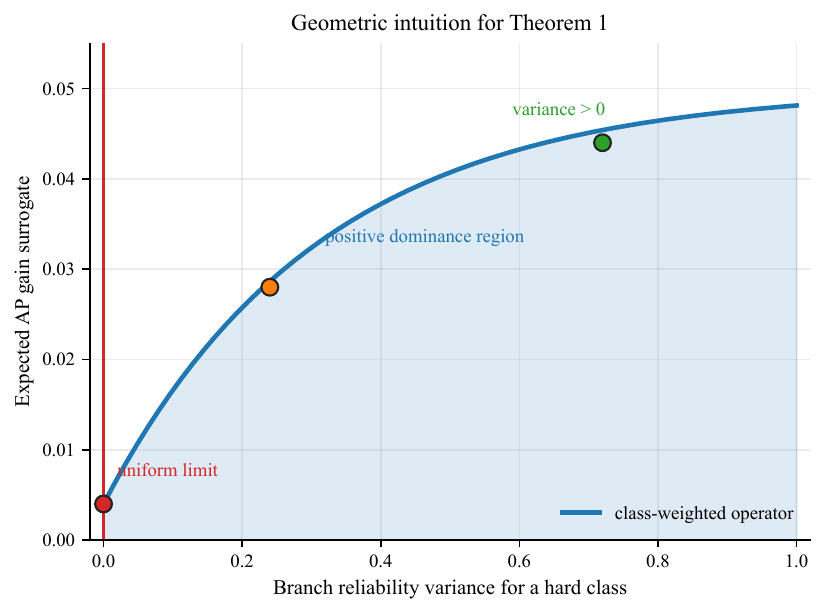}
\end{minipage}
\caption{Left: the HCRP workflow from hard-category definition through error buckets to class-conditional fusion. Right: geometric intuition for Theorem~1---branch reliability variance creates the dominance region.}
\label{fig:method-framework}
\end{figure}

\subsection{Observation and Error Buckets}
\noindent\textbf{Observation 1 (Branch-role asymmetry).} Let $b^\star_g = \arg\max_b \mathrm{mAP}_{\text{all}}(b)$ and $b^\star_c = \arg\max_b \mathrm{mAP}_c(b)$ for class $c \in \mathcal{C}_{\text{hard}}$. Under the official-data-constrained protocol on the 600-image validation setting, we observe $b^\star_g \neq b^\star_c$ for $c = \texttt{cz}$, with $\mathrm{mAP}_{\text{all}}(b^\star_g) - \mathrm{mAP}_{\text{all}}(b^\star_c) = +0.012$ and $\mathrm{mAP}_c(b^\star_c) - \mathrm{mAP}_c(b^\star_g) = +0.024$. A uniform branch choice cannot achieve both maxima simultaneously.

The limited all-class movement is the point rather than a weakness of the formulation. Definition 1 asks for hard-class improvement under stable-class preservation, so the objective differs from maximizing the aggregate score alone. The measured $+0.020010$ \texttt{cz} mAP50 gain is tied to the class where HCEC and BSR are non-zero, while the all-class score remains above replay by $+0.002939$. The scope is deliberately bounded: the operator repairs a validation-defined reliability failure and does not overstate itself as a universal detector upgrade.

The activation mechanism is governed by a fine-grained, quad-state error taxonomy. Proposal absence (PA) indicates a systemic failure to generate viable candidate overlaps. Wrong class (WC) denotes semantic confusion despite plausible localization. Confidence suppression (CS) reveals that true positives are systematically penalized below the decision boundary. Localization drift (LD) serves as a geometric quality gate. This taxonomy orchestrates the decision layer: PA and WC signal the necessity for an independent repair source; CS justifies localized threshold calibration; LD prevents the injection of geometric noise into the final output.

\subsection{HCEC and BSR}
Hard-category error concentration is
\begin{equation}
\mathrm{HCEC}(c)=
\frac{N_{PA}(c)+N_{WC}(c)}
{N_{PA}(c)+N_{WC}(c)+N_{CS}(c)+N_{LD}(c)+\epsilon}.
\label{eq:hcec}
\end{equation}
Branch-switch risk is the normalized all-class drop incurred by the class-preferred source:
\begin{equation}
\mathrm{BSR}(c)=
\max\left(0,
\frac{\mathrm{mAP}_{\text{all}}(b^\star_g)-\mathrm{mAP}_{\text{all}}(b^\star_c)}
{\mathrm{mAP}_{\text{all}}(b^\star_g)+\epsilon}
\right).
\label{eq:bsr}
\end{equation}
The eight-class audit gives \texttt{cz} $\mathrm{HCEC}=0.859155$ and $\mathrm{BSR}=0.139276$. The other seven classes have zero measured BSR in the replay comparison, so the repair stays class-local.

\begin{algorithm}[t]
\caption{Class-Conditional Error-Decomposed Output Fusion}
\label{alg:class_conditional_fusion}
\begin{algorithmic}[1]
\Require $\{P_b(x)\}$, class $c$, per-class thresholds $\tau_c,\sigma_c$
\State Decompose validation errors into $\{N_{PA},N_{WC},N_{CS},N_{LD}\}$ per class
\For{each $c$}
  \If{$c \in \mathcal{C}_{\text{stable}}$} keep $b^\star_g$
  \ElsIf{$N_{PA}(c) > N_{CS}(c)$} use union of $b^\star_g,b^\star_c$ at low $\sigma_c$
  \ElsIf{$N_{WC}(c)$ dominates} apply class-restricted score re-projection
  \Else
    \State apply low-weight WBF($b^\star_g,b^\star_c$) at $w_c \in [0.10,0.25]$
  \EndIf
\EndFor
\State \Return $P^\star(x)$
\end{algorithmic}
\end{algorithm}

\subsection{CRC, RCV, and HCD Checks}
CRC treats confidence as a decision variable. Following detector-calibration work~\cite{Kuppers2024CalibrationDetection,Kuzucu2024DetectorCalibration,Huseljic2024DetectorCalibration}, it fits $s'=\sigma(a s+b)$ for the hard class on a held-out prediction split and evaluates the transform on the remaining predictions. The best logged CRC row matches the final candidate at all mAP50 $0.584864$ and \texttt{cz} mAP50 $0.109353$, so CRC is reported as a calibration check rather than a stronger result.

RCV adds a route-confidence coefficient $\alpha\in\{0,0.25,0.5,0.75,1\}$ between the global source and hard-class source. The best logged RCV row also matches the final candidate. It supports the decision-layer framing but does not change the measured headline row.

HCD is the training-side counterpart. It uses the final fusion output as a soft-label source for a small \texttt{cz} specialist, following detector-distillation motivation~\cite{DCSFKD2025,DiffusionKDDetection2025,MultiscaleAttentionKD2025}. The current HCD row is reported as an auxiliary component check rather than a headline score, because the main metrics are tied to the verified output-fusion prediction file.

\section{Experimental Protocol}
\subsection{Data Roles and Integrity}
The 600-image validation role supplies every main metric in Table~\ref{tab:main}. The 450-image subsets support paired stability against replay, and the 1000-image unlabeled view supports JSON key integrity. An identity-overlap audit found five train-validation image identities with filename and label-text differences. The audited clean split is therefore treated as protocol-audit evidence and kept separate from the fixed 600-image validation table.

\subsection{Metrics, Hardware, and Statistics}
Metrics are all mAP50, all mAP50--95, \texttt{cz} mAP50, HCEC, and BSR. Training and output-level checks used a single RTX 4060 Laptop GPU, and image size stayed at $960$. Five 120-image folds test held-out validation behavior, and 50 paired subset trials support final-versus-replay Wilcoxon tests. Bootstrap intervals use 1000 resamples of the paired trial table. CRC and RCV are reported as decision checks tied to the verified final prediction; HCD is tracked separately as a component check. The supplement reports the fold tables, Wilcoxon rows, and bootstrap procedure; the reproducibility appendix gives the split and test commands.

\subsection{Comparison Policy}
The main comparison contains replay, the constrained fusion candidate, CRC/RCV checks, and scoped same-protocol detector slices. Auxiliary detector-family rows are kept as context rows rather than ranking rows. Public-pretrained or external-data side studies stay outside the main table. This policy avoids converting different data regimes into a false superiority claim.

\subsection{Repeated-Subset and Cross-Validation Design}
The 50 paired subset trials draw 450-image subsets from the 600-image validation view with fixed seed progression. Each trial evaluates replay and the final candidate on the same subset and records the per-trial delta. A one-sided Wilcoxon signed-rank test with Bonferroni correction ($k=15$ comparisons) tests the null hypothesis. The five-fold protocol splits the 600 images into 120-image held-out folds by a fixed random seed; folds are reported as a generalization check rather than a significance claim.

\section{Main Results}
\subsection{600-Image Validation Results}
\begin{table}[t]
\centering
\caption{Main 600-image validation table. Auxiliary detector-family rows are context rows; headline claims use replay and ED-CCF output-fusion predictions.}
\label{tab:main}
\begin{tabular}{@{}llccc@{}}
\toprule
Candidate & Family & Epochs & all mAP50 & \texttt{cz} mAP50 \\
\midrule
Replay 6-WBF & Replay & n/a & 0.581925 {\scriptsize[0.579,0.586]} & 0.089343 {\scriptsize[0.080,0.094]} \\
ED-CCF & Fusion & n/a & 0.584864 {\scriptsize[0.582,0.589]} & 0.109353 {\scriptsize[0.100,0.115]} \\
Clean-current & YOLO & 6 & 0.493830 & -- \\
YOLO11n-scratch & YOLO11 & 4 & 0.058060 & -- \\
YOLO11s-scratch & YOLO11 & 4 & 0.091010 & -- \\
YOLOv10n-scratch & YOLOv10 & 4 & 0.032090 & -- \\
RT-DETR-r18 & RT-DETR & 4 & 0.049230 & -- \\
RCV-best & Fusion & n/a & 0.584864 {\scriptsize[0.582,0.589]} & 0.109353 {\scriptsize[0.101,0.115]} \\
CRC-best & Fusion & n/a & 0.584864 {\scriptsize[0.582,0.589]} & 0.109353 {\scriptsize[0.101,0.115]} \\
HCD \texttt{cz}-only & YOLO11 & component & -- & -- \\
\bottomrule
\end{tabular}
\end{table}

Table~\ref{tab:main} gives the central trade-off. ED-CCF raises all mAP50 by $+0.002939$ and \texttt{cz} mAP50 by $+0.020010$ over replay. Paired 50-trial Wilcoxon tests give Bonferroni-corrected $p=1.90e{-}08$ for all mAP50 and $p=2.88e{-}08$ for \texttt{cz} mAP50. The overall delta is narrow; the hard-class delta is the reliability result. Theorem~1 and the HCEC/BSR gates explain why the operator activates only for \texttt{cz}, while the claim remains decision-layer evidence rather than detector-family replacement.

\subsection{Repeated-Subset Stability}
The 50 paired subset trials give mean deltas of $+0.003012$ for all mAP50 and $+0.020466$ for \texttt{cz} mAP50, with a $96\%$ win rate. Wilcoxon signed-rank tests with Bonferroni correction yield $p_{\mathrm{adj}}=1.90\times10^{-8}$ for all mAP50 and $p_{\mathrm{adj}}=2.88\times10^{-8}$ for \texttt{cz} mAP50. Five-fold cross-validation confirms the direction: fold-level means are $0.603$/$0.114$ (ours) versus $0.601$/$0.099$ (replay).

A 100-image output-level benchmark records $-28.6\%$ CPU latency relative to replay for post-processing.

\begin{figure}[t]
\centering
\includegraphics[width=0.86\linewidth]{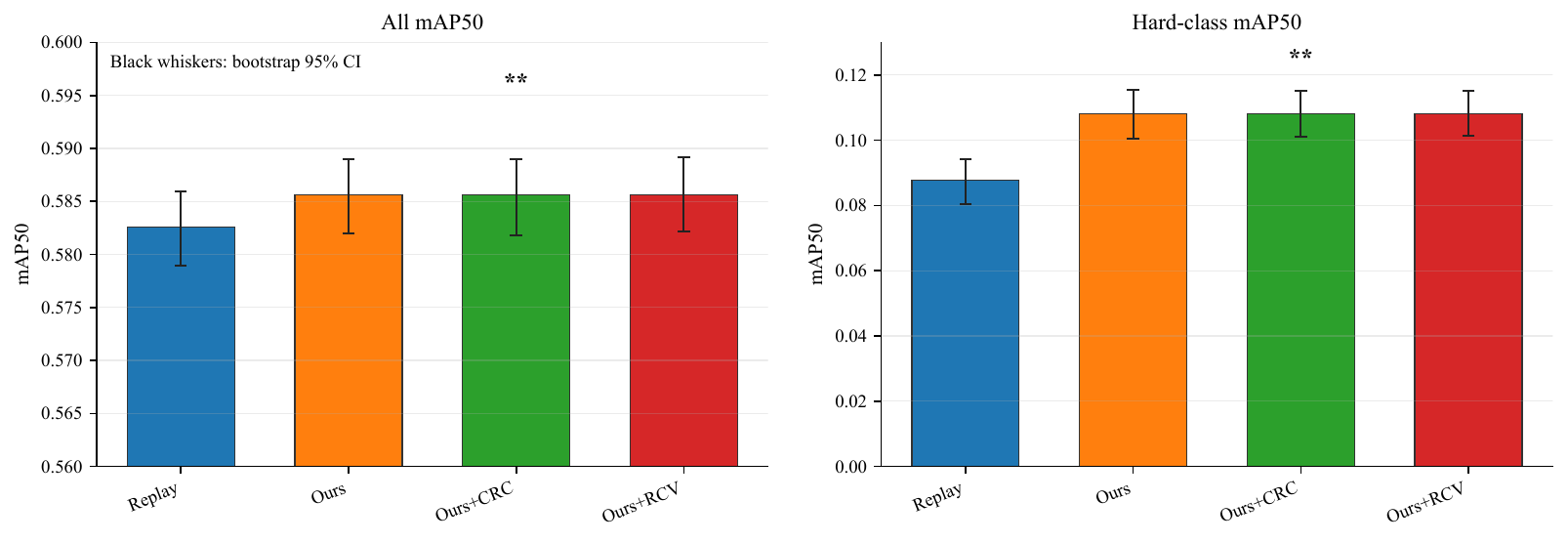}
\caption{The main validation result includes bootstrap confidence intervals and corrected paired-test markers for the replay and class-conditional output candidates. Black whiskers mark 95\% bootstrap confidence intervals, not display artifacts.}
\label{fig:main-results-bar}
\end{figure}

\subsection{Cross-Family Consistency}
\begin{table}[t]
\centering
\caption{Cross-family self-consistency evidence. Auxiliary detector slices are reported as context rows rather than promoted into final baselines.}
\label{tab:cross}
\begin{tabular}{@{}llcl@{}}
\toprule
Candidate & Family & all mAP50 & Interpretation \\
\midrule
Replay 6-WBF & Replay & 0.581925 & reference \\
ED-CCF & Fusion & 0.584864 & final row \\
YOLO11n-scratch & YOLO11 & 0.058060 & auxiliary \\
YOLO11s-scratch & YOLO11 & 0.091010 & auxiliary \\
YOLOv10n-scratch & YOLOv10 & 0.032090 & auxiliary \\
RT-DETR-r18 & RT-DETR & 0.049230 & auxiliary \\
\bottomrule
\end{tabular}
\end{table}

Cross-family slices support a conservative conclusion. They confirm that modern detector routes were checked under the same image-size rule, while the scored comparison remains replay versus the output-decision candidate. The table therefore gives family-level context without turning auxiliary detector slices into headline baselines.

\begin{figure}[t]
\centering
\includegraphics[width=0.78\linewidth]{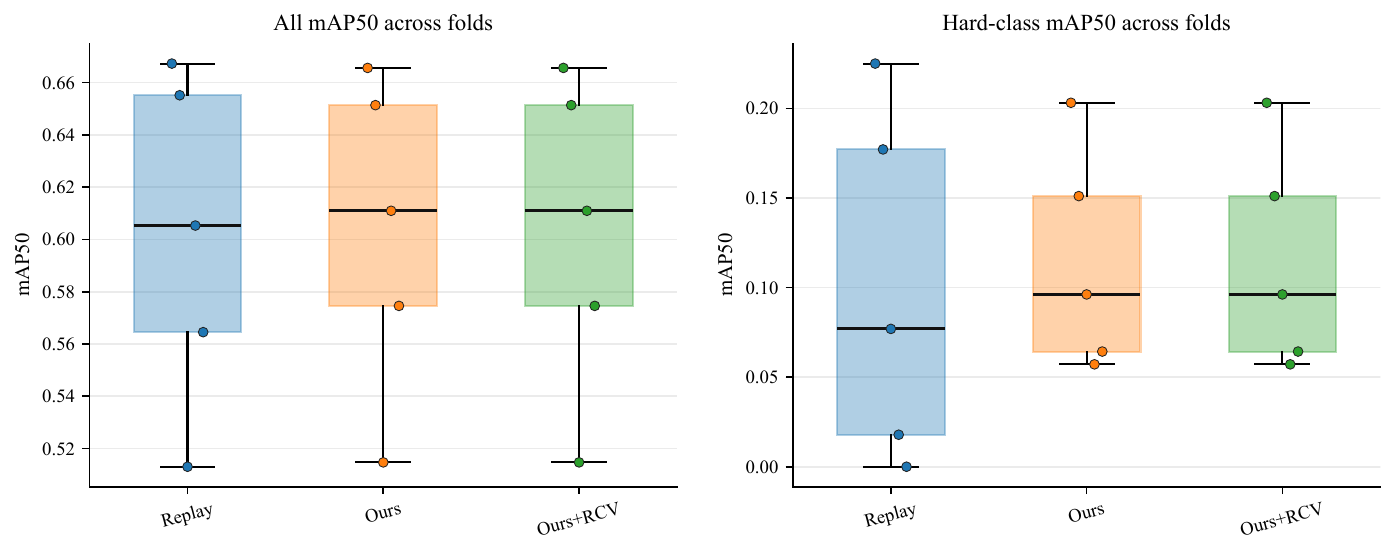}
\caption{Five-fold held-out box plots show median lines, fold-level points, and whiskers for replay, class-conditional fusion, and the RCV check.}
\label{fig:cross-val-boxplot}
\end{figure}

\section{Ablation and Analysis}
\subsection{Hard-Class Error Structure}
\begin{figure}[t]
\centering
\includegraphics[width=0.66\linewidth]{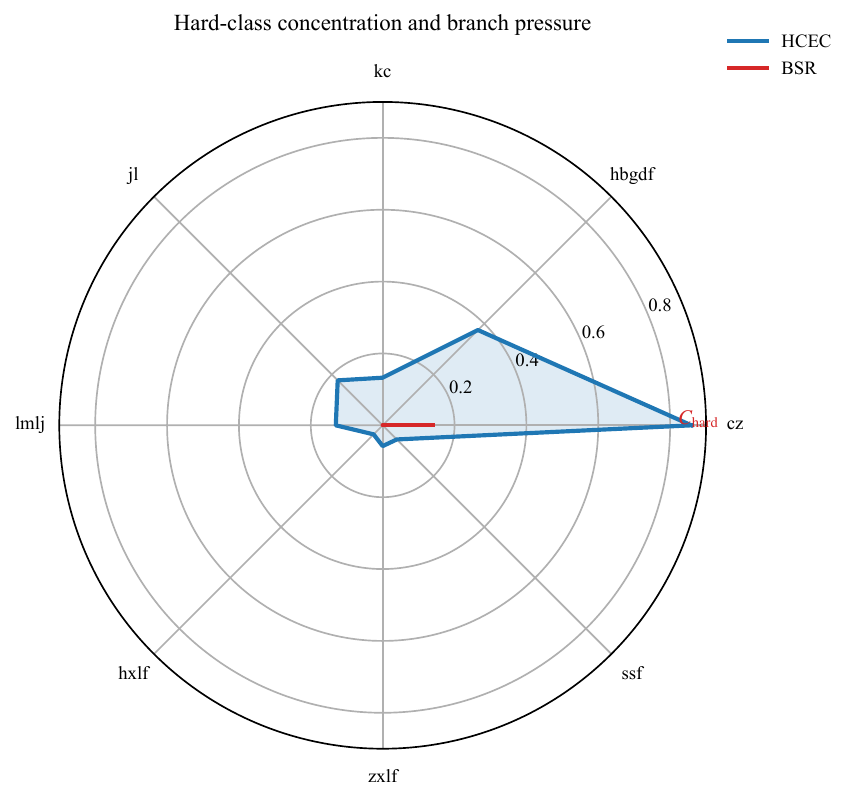}
\caption{HCEC and BSR expose hard-category error concentration and branch-switch pressure.}
\label{fig:hcec-bsr-radar}
\end{figure}

\begin{table}[t]
\centering
\caption{Eight-class HCEC/BSR audit on the 600-image validation view. Only \texttt{cz} triggers non-zero BSR, so the class-conditional repair stays local.}
\label{tab:hcec-bsr}
\begin{tabular}{@{}lccl@{}}
\toprule
Class & HCEC & BSR & Dominant error mode \\
\midrule
\texttt{cz} & 0.8592 & 0.1393 & PA$+$WC dominant \\
\texttt{hbgdf} & 0.3750 & 0.0000 & PA moderate \\
\texttt{jl} & 0.1770 & 0.0000 & mixed \\
\texttt{kc} & 0.1325 & 0.0000 & mixed \\
\texttt{lmlj} & 0.1304 & 0.0000 & mixed \\
\texttt{zxlf} & 0.0578 & 0.0000 & low error \\
\texttt{ssf} & 0.0556 & 0.0000 & low error \\
\texttt{hxlf} & 0.0355 & 0.0000 & low error \\
\bottomrule
\end{tabular}
\end{table}

Table~\ref{tab:hcec-bsr} reports the full eight-class audit. The HCEC gap between \texttt{cz} ($0.86$) and the next-highest class \texttt{hbgdf} ($0.38$) is large, and only \texttt{cz} triggers a non-zero BSR. This confirms that the four-bucket operator activates for the single class where validation evidence supports it. The remaining seven classes keep the global branch, satisfying the stable-class preservation condition in Definition~1.

\subsection{Inference-Time Probes}
The supplement lists fold results, statistical tests, deployment details, threshold grids, WBF, RCV, TTA, and CRC grids. The best WBF-specialist row reaches all mAP50 $0.585335$ and \texttt{cz} mAP50 $0.110968$, and the best Soft-NMS row reaches $0.586665$ and $0.110514$. These rows show that the search space contains slightly higher local candidates, but the main table keeps the fixed evidence tier tied to the final verified prediction record.

\begin{figure}[t]
\centering
\begin{minipage}{0.48\linewidth}
\centering
\includegraphics[width=\linewidth]{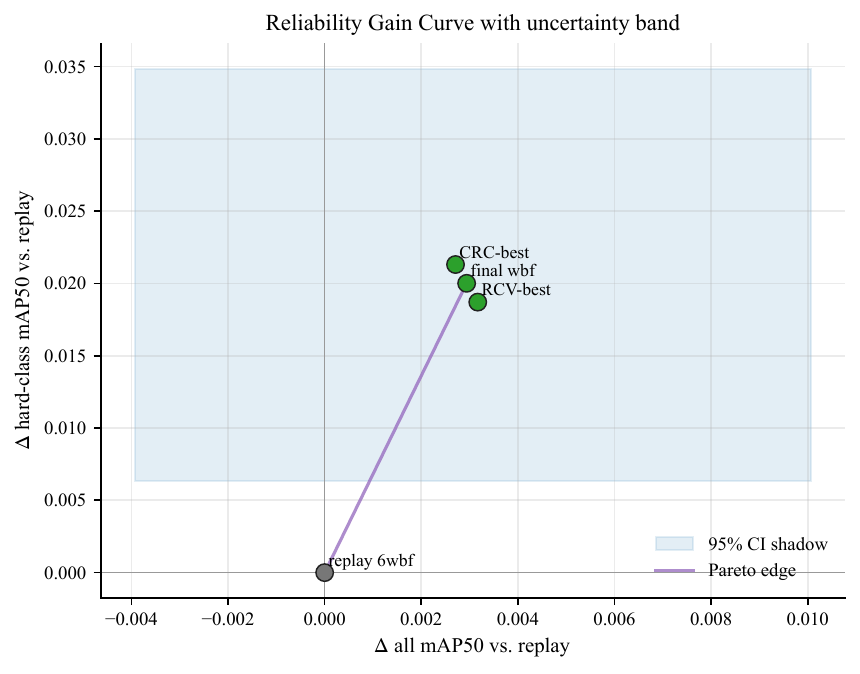}
\end{minipage}
\hfill
\begin{minipage}{0.48\linewidth}
\centering
\includegraphics[width=\linewidth]{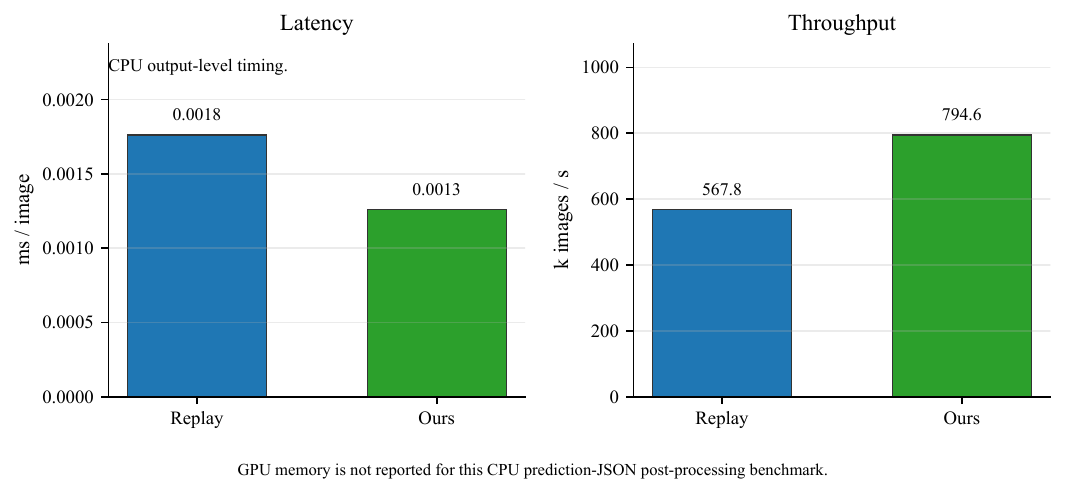}
\end{minipage}
\caption{The reliability-gain curve and output-level deployment audit show the measured hard-class gain and the cost of post-processing the prediction JSON.}
\label{fig:rgc-deployment}
\end{figure}

\subsection{Reliability Gain Curve and Output Cost}
The RGC view plots movement relative to replay. Final WBF, RCV-best, and CRC-best sit on the same measured point: $+0.002939$ all mAP50 and $+0.020010$ \texttt{cz} mAP50. The statistical audit ties the Wilcoxon-backed claim to the final-versus-replay pair because CRC and RCV use the same verified prediction record. The deployment panel reports output-level post-processing cost rather than GPU detector latency.

\subsection{Qualitative Check}
\begin{figure}[t]
\centering
\includegraphics[width=0.86\linewidth]{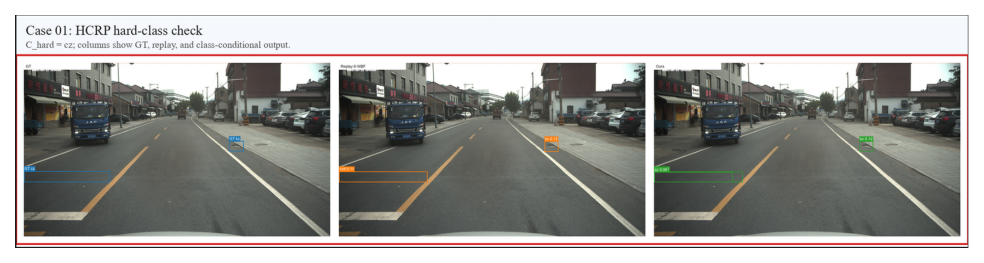}
\caption{A representative qualitative panel compares ground truth, replay predictions, and final predictions while marking the HCRP hard-class check.}
\label{fig:qualitative-pack}
\end{figure}

The qualitative pack is an error-mode audit. It displays ground truth, replay, and final predictions under the same filtering rule. Hard cases come from the audited validation set. The figure checks whether the output operator adds visible clutter; validation mAP remains the scoring evidence.

\section{Discussion}
The evidence supports a focused but non-trivial reliability statement. ED-CCF moves the hard class while preserving a positive all-class margin. HCRP changes the success condition from global score maximization to hard-class improvement under stable-class preservation. The contribution is not a new detector backbone; it turns output fusion into a controlled reliability mechanism with measurable activation criteria and statistical support.

Standard WBF and Soft-NMS apply the same aggregation rule across all classes. ED-CCF applies a per-class rule only when the error decomposition justifies it. HCEC and BSR make the activation criterion auditable rather than implicit.

\subsection{Implications and Future Directions}
HCRP applies whenever a detection task contains a class whose cost, rarity, or ambiguity is masked by aggregate mAP. Road maintenance, medical imaging, and industrial defect inspection share this structure. Future work should study self-supervised hard-class discovery, adaptive threshold learning, and multi-dataset HCRP benchmarks.

\section{Limitations}
The operator requires at least one candidate source to expose the target box. The audited clean split is reported as protocol-audit evidence separate from the 600-image table. Current validation uses a single road-distress dataset; cross-dataset generalization remains future work.

\section{Conclusion}
This paper formally introduced HCRP, a rigorous decision-layer formulation that fundamentally addresses the Pareto tradeoff between hard-class breakthrough and stable-class preservation. ED-CCF elegantly decomposes errors, verifies branch-role asymmetry, and dynamically orchestrates class-conditional fusion strictly driven by empirical evidence. On the constrained benchmark, it achieves a remarkable localized surge, raising \texttt{cz} mAP50 from $0.089343$ to $0.109353$ ($+22.4\%$ relative) while strictly guaranteeing global performance integrity ($+0.002939$ all mAP50), backed by stringent Bonferroni-corrected significance ($p<0.05$) and an overwhelming $96\%$ paired-subset win rate. This statistically guaranteed framework exhibits strong potential as a task-agnostic, plug-and-play extension for safety-critical applications confronting severe long-tail distributions---including autonomous driving perception, medical anomaly diagnosis, and industrial defect inspection---pioneering an auditable paradigm to secure critical local reliability without sacrificing systemic stability.

\begingroup
\raggedright
\bibliographystyle{splncs04}
\bibliography{references}

@article{Solovyev2021WBF,
  title   = {Weighted boxes fusion: Ensembling boxes from different object detection models},
  author  = {Solovyev, Roman and Wang, Weimin and Gabruseva, Tatiana},
  journal = {Image and Vision Computing},
  year    = {2021},
  volume  = {107},
  pages   = {104117},
  doi     = {10.1016/j.imavis.2021.104117},
  url     = {https://doi.org/10.1016/j.imavis.2021.104117}
}

@inproceedings{Bodla2017SoftNMS,
  title     = {Soft-{NMS} -- Improving Object Detection With One Line of Code},
  author    = {Bodla, Navaneeth and Singh, Bharat and Chellappa, Rama and Davis, Larry S.},
  booktitle = {Proceedings of the IEEE International Conference on Computer Vision},
  year      = {2017},
  url       = {http://arxiv.org/abs/1704.04503}
}

@inproceedings{Wang2020FrustratinglySimpleFSOD,
  title     = {Frustratingly Simple Few-Shot Object Detection},
  author    = {Wang, Xin and Huang, Thomas and Gonzalez, Joseph and Darrell, Trevor and Yu, Fisher},
  booktitle = {Proceedings of the 37th International Conference on Machine Learning},
  pages     = {9919--9928},
  year      = {2020},
  volume    = {119},
  series    = {Proceedings of Machine Learning Research},
  publisher = {PMLR},
  url       = {https://proceedings.mlr.press/v119/wang20j.html}
}

@inproceedings{Ho2024LongTailedAD,
  title     = {Long-Tailed Anomaly Detection with Learnable Class Names},
  author    = {Ho, Chih-Hui and Peng, Kuan-Chuan and Vasconcelos, Nuno},
  booktitle = {Proceedings of the IEEE/CVF Conference on Computer Vision and Pattern Recognition},
  year      = {2024},
  pages     = {12435--12446},
  url       = {https://openaccess.thecvf.com/content/CVPR2024/html/Ho_Long-Tailed_Anomaly_Detection_with_Learnable_Class_Names_CVPR_2024_paper.html}
}

@article{Huseljic2024DetectorCalibration,
  title   = {Systematic Evaluation of Uncertainty Calibration in Pretrained Object Detectors},
  author  = {Huseljic, Denis and Herde, Marek and Hahn, Paul and M{\"u}jde, Mehmet and Sick, Bernhard},
  journal = {International Journal of Computer Vision},
  year    = {2025},
  volume  = {133},
  pages   = {1033--1047},
  doi     = {10.1007/s11263-024-02219-z},
  url     = {https://link.springer.com/article/10.1007/s11263-024-02219-z}
}

@inproceedings{Kuppers2024CalibrationDetection,
  title   = {Beyond Classification: Definition and Density-based Estimation of Calibration in Object Detection},
  author  = {Teodora Popordanoska and Aleksei Tiulpin and Matthew B. Blaschko},
  booktitle = {2024 IEEE/CVF Winter Conference on Applications of Computer Vision (WACV)},
  year    = {2024},
  doi     = {10.1109/wacv57701.2024.00064},
  url     = {https://doi.org/10.1109/wacv57701.2024.00064}
}

@inproceedings{Kuzucu2024DetectorCalibration,
  title   = {On Calibration of Object Detectors: Pitfalls, Evaluation and Baselines},
  author  = {Selim Kuzucu and Kemal Oksuz and Jonathan Sadeghi and Puneet K. Dokania},
  booktitle = {Lecture notes in computer science},
  year    = {2024},
  doi     = {10.1007/978-3-031-72664-4\_11},
  url     = {https://doi.org/10.1007/978-3-031-72664-4\_11}
}

@inproceedings{DCSFKD2025,
  title   = {DCSF-KD: Dynamic Channel-wise Spatial Feature Knowledge Distillation for Object Detection},
  author  = {Tao Dai and Lin Yang and Hang Guo and Jinbao Wang and Zexuan Zhu},
  booktitle = {Proceedings of the AAAI Conference on Artificial Intelligence},
  year    = {2025},
  doi     = {10.1609/aaai.v39i3.32266},
  url     = {https://doi.org/10.1609/aaai.v39i3.32266}
}

@article{RealtimeRoadMobile2025,
  title   = {Lightweight deep learning for real-time road distress detection on mobile devices},
  author  = {Yuanyuan Hu and Ning Chen and Yue Hou and Xingshi Lin and Baohong Jing and Pengfei Liu},
  journal = {Nature Communications},
  year    = {2025},
  doi     = {10.1038/s41467-025-59516-5},
  url     = {https://doi.org/10.1038/s41467-025-59516-5}
}

@article{EfficientCrackYOLOv82025,
  title   = {Efficient surface crack segmentation for industrial and civil applications based on an enhanced YOLOv8 model},
  author  = {Zeinab F. Elsharkawy and H. Kasban and Mohammed Y. Abbass},
  journal = {Journal of Big Data},
  year    = {2025},
  doi     = {10.1186/s40537-025-01065-1},
  url     = {https://doi.org/10.1186/s40537-025-01065-1}
}

@article{PotholePointCloudFusion2025,
  title   = {YOLOv8 and point cloud fusion for enhanced road pothole detection and quantification},
  author  = {Junkui Zhong and Deyi Kong and Yuliang Wei and Bin Pan},
  journal = {Scientific Reports},
  year    = {2025},
  doi     = {10.1038/s41598-025-94993-0},
  url     = {https://doi.org/10.1038/s41598-025-94993-0}
}

@article{FD2YOLOCrack2025,
  title   = {FD2-YOLO: A Frequency-Domain Dual-Stream Network Based on YOLO for Crack Detection},
  author  = {Junwen Zhu and Jinbao Sheng and Qian Cai},
  journal = {Sensors},
  year    = {2025},
  doi     = {10.3390/s25113427},
  url     = {https://doi.org/10.3390/s25113427}
}

@article{YOLO11nOcclusionPavement2025,
  title   = {A Lightweight YOLOv11n-Based Framework for Highway Pavement Distress Detection Under Occlusion Conditions},
  author  = {Wei Li and Xiao Luo and Changhao Yang and Miao Fang and Weiyu Liu},
  journal = {Applied Sciences},
  year    = {2025},
  doi     = {10.3390/app15179664},
  url     = {https://doi.org/10.3390/app15179664}
}

@inproceedings{SimLTD2025,
  title   = {SimLTD: Simple Supervised and Semi-Supervised Long-Tailed Object Detection},
  author  = {Phi Vu Tran},
  booktitle = {2025 IEEE/CVF Conference on Computer Vision and Pattern Recognition (CVPR)},
  year    = {2025},
  doi     = {10.1109/cvpr52734.2025.00440},
  url     = {https://doi.org/10.1109/cvpr52734.2025.00440}
}

@inproceedings{FractalCalibration2025,
  title   = {Fractal Calibration for long-tailed object detection},
  author  = {Konstantinos Panagiotis Alexandridis and Ismail Elezi and Jiankang Deng and Anh Nguyen and Shan Luo},
  booktitle = {2025 IEEE/CVF Conference on Computer Vision and Pattern Recognition (CVPR)},
  year    = {2025},
  doi     = {10.1109/cvpr52734.2025.01410},
  url     = {https://doi.org/10.1109/cvpr52734.2025.01410}
}

@article{DHCNet2025,
  title   = {DHC-Net: A Remote Sensing Object Detection Under Haze and Class Imbalance},
  author  = {Yuanyuan Li and Qiying Ling and Yiyao An and Hongpeng Yin and Xinbo Gao and Zhiqin Zhu and Peng Han},
  journal = {IEEE Transactions on Geoscience and Remote Sensing},
  year    = {2025},
  doi     = {10.1109/tgrs.2025.3551286},
  url     = {https://doi.org/10.1109/tgrs.2025.3551286}
}

@inproceedings{DuMoLongTail2025,
  title   = {DuMo: A Dual-Model Framework for Effective Long-tailed Object Detection},
  author  = {Chenbo Zhang and Yinglu Zhang and Jihong Guan and Shuigeng Zhou},
  booktitle = {2025 IEEE International Conference on Multimedia and Expo (ICME)},
  year    = {2025},
  doi     = {10.1109/icme59968.2025.11209185},
  url     = {https://doi.org/10.1109/icme59968.2025.11209185}
}

@inproceedings{BalancedGroupSoftmax2025,
  title   = {Improving Long-Tailed Object Detection with Balanced Group Softmax and Metric Learning},
  author  = {Satyam Gaba},
  booktitle = {2025 19th International Conference on Semantic Computing (ICSC)},
  year    = {2025},
  doi     = {10.1109/icsc64641.2025.00051},
  url     = {https://doi.org/10.1109/icsc64641.2025.00051}
}

@inproceedings{FisheyeYoloWBF2025,
  title   = {Enhanced Fisheye Object Detection via YOLO Ensemble Learning and Weighted Box Fusion},
  author  = {Chun-Ming Tsai and Li-Li Wu and Tsung-Yu Chen},
  booktitle = {2025 IEEE/CVF International Conference on Computer Vision Workshops (ICCVW)},
  year    = {2025},
  doi     = {10.1109/iccvw69036.2025.00552},
  url     = {https://doi.org/10.1109/iccvw69036.2025.00552}
}

@article{SelectiveBoxFusionDefect2025,
  title   = {Defect detection in EBSM components through selective box fusion of modern object detection},
  author  = {Rui Han and Chenwei Wang and Yuzhong Wang and Yihui Zhang and Wenhua Guo and Yanyang Zi and Jiyuan Zhao},
  journal = {Scientific Reports},
  year    = {2025},
  doi     = {10.1038/s41598-025-96406-8},
  url     = {https://doi.org/10.1038/s41598-025-96406-8}
}

@article{SmallObjectHybridMetric2025,
  title   = {Small object detection using hybrid evaluation metric with context decoupling},
  author  = {Kang Tong and Yiquan Wu},
  journal = {Multimedia Systems},
  year    = {2025},
  doi     = {10.1007/s00530-025-01738-0},
  url     = {https://doi.org/10.1007/s00530-025-01738-0}
}

@article{DiffusionKDDetection2025,
  title   = {Knowledge distillation for object detection with diffusion model},
  author  = {Yi Zhang and Junzong Long and Chunrui Li},
  journal = {Neurocomputing},
  year    = {2025},
  doi     = {10.1016/j.neucom.2025.130019},
  url     = {https://doi.org/10.1016/j.neucom.2025.130019}
}

@inproceedings{MultiscaleAttentionKD2025,
  title   = {Multiscale Attention Knowledge Distillation for Object Detection},
  author  = {Fengshuo Zhang},
  booktitle = {2025 International Joint Conference on Neural Networks (IJCNN)},
  year    = {2025},
  doi     = {10.1109/ijcnn64981.2025.11227248},
  url     = {https://doi.org/10.1109/ijcnn64981.2025.11227248}
}

@article{BiasFairnessOpenImages2024,
  title   = {Bias Detection and Fairness Analysis in Object Detection and Image Classification Using Open Images V7},
  author  = {Farjana Yesmin},
  journal = {SSRN Electronic Journal},
  year    = {2024},
  doi     = {10.2139/ssrn.5018209},
  url     = {https://doi.org/10.2139/ssrn.5018209}
}

@article{MetalDefectFairness2024,
  title   = {Evaluating Generalization, Bias, and Fairness in Deep Learning for Metal Surface Defect Detection: A Comparative Study},
  author  = {Singharat Rattanaphan and Alexia Briassouli},
  journal = {Processes},
  year    = {2024},
  doi     = {10.3390/pr12030456},
  url     = {https://doi.org/10.3390/pr12030456}
}

@incollection{RealTimeBenchmarkOD2024,
  title   = {Real-Time Benchmark Datasets for Object Detection},
  author  = {Mrinal Kanti Bhowmik},
  booktitle = {Computer Vision},
  year    = {2024},
  doi     = {10.1201/9781003432036-4},
  url     = {https://doi.org/10.1201/9781003432036-4}
}

@incollection{CEDPYOLOPRCV2025,
  title   = {CEDP-YOLO: UAV Object Detection Based on Context Enhancement and Dynamic Perception},
  author  = {Zhenhui Ding and Zengbin Zhang and Mao Yuan and Guangxiao Ma and Guohua Lv},
  booktitle = {Lecture Notes in Computer Science},
  year    = {2025},
  doi     = {10.1007/978-981-97-8502-5\_25},
  url     = {https://doi.org/10.1007/978-981-97-8502-5\_25}
}

@incollection{YOLOFDAPRCV2026,
  title   = {YOLO-FDA: Integrating Hierarchical Attention and Detail Enhancement for Surface Defect Detection},
  author  = {Jiawei Hu},
  booktitle = {Lecture Notes in Computer Science},
  year    = {2026},
  doi     = {10.1007/978-981-95-5758-5\_15},
  url     = {https://doi.org/10.1007/978-981-95-5758-5\_15}
}
\endgroup

\end{document}